\documentclass[journal,twoside,web]{ieeecolor}

\usepackage{graphics} 
\usepackage{epsfig} 
\usepackage{mathptmx} 
\usepackage{amsmath} 
\usepackage{amssymb}  
\usepackage{xcolor}

\usepackage{lcsys}

\newtheorem{theorem}{Theorem}[section]

\newtheorem{lemma}[theorem]{Lemma}

\newtheorem{definition}[theorem]{Definition}
\newtheorem{assumption}[theorem]{Assumption}
\newtheorem{remark}[theorem]{Remark}

\usepackage{algorithm}
\usepackage{algorithmic}

\newcommand{\man}{\mathit{M}}
\newcommand{\manspd}{\mathrm{SPD}(n)}
\newcommand{\R}{\mathbb{R}}
\newcommand{\tansp}[1]{T_{#1} \man}
\newcommand{\tanspspd}[1]{T_{#1} \manspd}
\newcommand{\mansubset}{\mathit{X}}
\newcommand{\proj}{P_{\mansubset}}
\newcommand{\norm}[1]{\big \lVert #1 \big \rVert}

\newcommand{\rgrad}[1]{\text{grad} #1}
\newcommand{\expm}[1]{\text{Exp}_{#1}}
\newcommand{\logm}[1]{\text{Log}_{#1}}
\newcommand{\innerprod}[2]{\langle #1 , #2 \rangle}

\newcommand{\xopt}{x^*}
\newcommand{\regret}{\text{Regret}_T}
\newcommand{\sumt}{\sum_{t=1}^T}

\newcommand{\argmin}{\text{argmin}}

\begin{document}
\title{
Online Optimization on Hadamard Manifolds: Curvature Independent Regret Bounds on Horospherically Convex Objectives 
}
\author{Emre Sahinoglu and Shahin Shahrampour
\thanks{This work is partially supported by NSF under Awards ECCS-2240788 and ECCS-2442321.}
\thanks{Emre Sahinoglu and Shahin Shahrampour are with the Department of Mechanical \& Industrial Engineering,
        Northeastern University, Boston, MA 02115
        {\tt\small \ Email: \{sahinoglu.m \& s.shahrampour\}@northeastern.edu}.}%
}

\maketitle
\thispagestyle{empty}

\begin{abstract}

We study online Riemannian optimization on Hadamard manifolds under the framework of horospherical convexity (h-convexity). Prior work mostly relies on the geodesic convexity (g-convexity), leading to regret bounds scaling poorly with the manifold curvature. To address this limitation, we analyze Riemannian online gradient descent for h-convex and strongly h-convex functions and establish  $O(\sqrt{T})$ and $O(\log(T))$ regret guarantees, respectively. These bounds are curvature-independent and match the results in the Euclidean setting. We validate our approach with experiments on the manifold of symmetric positive definite (SPD) matrices equipped with the affine-invariant metric. In particular, we investigate online Tyler's $M$-estimation and online Fréchet mean computation, showing the application of h-convexity in practice. 
\end{abstract}


\section{Introduction}

Riemannian optimization has emerged as a powerful framework for solving constrained optimization problems where the feasible set naturally possesses a geometric structure. Unlike Euclidean methods that rely on projections to maintain feasibility, Riemannian approaches exploit the intrinsic geometry of the manifold to design algorithms that are both efficient and theoretically principled \cite{absil2009optimization}. In control theory, Riemannian optimization has been applied to the Linear–Quadratic Regulator (LQR) problem \cite{talebi2023policy}, system identification \cite{sato2019riemannian1}, consensus and synchronization on nonlinear state spaces \cite{sarlette2009consensus}, and covariance estimation on the manifold of positive-definite matrices \cite{moakher2005differential}. In machine learning, manifold-constrained optimization problems arise in low-rank matrix completion \cite{vandereycken2013low}, principal component analysis \cite{edelman1998geometry}, dictionary learning \cite{cherian2016riemannian}, and representation learning in hyperbolic spaces \cite{nickel2017poincare}. By aligning the optimization algorithms with the underlying geometry, Riemannian optimization not only ensures feasibility but also uncovers algorithmic principles that are central to modern large-scale learning and control applications.

While Riemannian optimization has been extensively studied in offline settings, many modern applications require algorithms that adapt to sequential data. This motivates the study of {\it online Riemannian optimization}, where at each round a learner selects a point on a subset of the manifold and subsequently observes a loss function defined on that point \cite{wang2025online,wang2023online}. A common performance metric for this setting is {\it regret}, which compares the cumulative loss of the online sequence with that of the best decision in hindsight. Online Riemannian optimization extends the classical theory of online convex optimization \cite{hazan2016introduction} to non-Euclidean spaces, introducing new challenges stemming from curvature, nonlinearity, and the lack of global linear structure. These challenges are particularly relevant in adaptive control, distributed learning, and streaming machine learning tasks, where the decision variable belongs to various manifolds, such as Lie groups, spheres, or the set of positive-definite matrices \cite{cherian2016riemannian} ,\cite{lezcano2019cheap,  sato2023riemannian}.

One central challenge in extending the optimization theory to manifolds is the proper generalization of convexity. The natural analogue of Euclidean convexity along line segments is the geodesic convexity (g-convexity), quantifying convexity along geodesics \cite{zhang2016first}. While this notion provides a useful foundation for convergence analysis, it also introduces nontrivial geometric complexities since classical inequalities, such as the law of cosines, take curvature-dependent forms that affect the attainable convergence rates. On Hadamard manifolds, simply connected spaces with non-positive curvature, curvature-adjusted cosine or comparison inequalities replace Euclidean quadratic expansions, rendering convergence guarantees more delicate \cite{zhang2016first}. These curvature-induced corrections highlight the fundamental differences between Euclidean and Riemannian settings, where regret bounds and convergence rates deteriorate due to curvature-dependent constants, involving both sectional curvature and the diameter of the feasible set. This motivates the exploration of alternative convexity notions, such as horospherical convexity (h-convexity), that more effectively capture the asymptotic geometric behavior in non-positively curved spaces \cite{criscitiello2025horospherically,goodwin2024subgradient,hirai2024convex,lewis2024horoballs}.

In this letter, we study Riemannian online optimization for h-convex and strongly h-convex objectives, characterized by Busemann functions as a natural extension of affine functions in Euclidean spaces. While recent work on g-convex objectives establishes regret bounds that scale poorly with respect to the lower bound of the sectional curvature, we aim for curvature-independent guarantees by focusing on h-convex objectives instead. In what follows, 
\begin{itemize}
    \item Section \ref{sec:background} reviews the background on online optimization and convexity notions on manifolds. We discuss both inner and outer characterizations of convexity on Hadamard manifolds and introduce h-convex objectives, defined via Busemann functions. In addition, we present a geometric lemma that establishes an inequality for geodesic triangles, which plays a key role in deriving the regret bound for strongly h-convex objectives.
    \item Section \ref{sec:theoretical} presents the technical assumptions as well as the Riemannian Online Gradient Descent (ROGD) algorithm. We analyze the regret rates, establishing $O(\sqrt{T})$ regret for h-convex objectives and $O(\log (T))$ regret for strongly h-convex objectives. Notably, these bounds match the corresponding Euclidean regret rates and do {\it not} involve curvature-dependent terms.
    
    \item Section \ref{sec:numerical} illustrates the performance of ROGD via numerical experiments. We consider the manifold of symmetric positive-definite (SPD) matrices, endowed with the affine-invariant metric. Our testbeds are the Tyler’s $M$-estimator problem for h-convex objectives and the Fr\'echet mean estimation for strongly h-convex objectives. Consistent with our theoretical claims, the numerical results demonstrate that ROGD achieves sublinear regret in both cases.
\end{itemize}

\subsection{Literature Review}

\textbf{Riemannian Optimization.} 
Over the past two decades, there has been a growing interest in Riemannian optimization, motivated by the prevalence of problems with non-Euclidean constraints in machine learning, signal processing, and control \cite{absil2009optimization, edelman1998geometry,sato2021riemannian}. Many classical Euclidean algorithms have been successfully extended to the Riemannian setting. Foundational work formalized retractions, vector transports, and convergence guarantees for first-order and second-order methods \cite{boumal2023introduction}, while stochastic  extensions such as Riemannian SGD \cite{bonnabel2013stochastic}, variance-reduced methods \cite{sato2019riemannian, zhang2016riemannian}, and 
decentralized techniques \cite{chen2021decentralized,sun2024global} provided practical tools for optimization on curved spaces. To capture convexity in the Riemannian setting, the notion of g-convexity was introduced, enabling the complexity analysis of first-order methods, mirroring Euclidean convergence results with the addition of curvature-dependent terms \cite{zhang2016first, vishnoi2018geodesic}. More recently, an alternative notion of convexity, i.e., h-convexity, has been explored, motivated by outer characterizations of convexity via Busemann functions and horospheres, leading to a new framework for optimization in non-positively curved Hadamard manifolds \cite{criscitiello2025horospherically,goodwin2024subgradient,lewis2024horoballs}.

\textbf{Online Riemannian Optimization.} The extension of Euclidean online convex optimization (OCO) to manifolds has gained significant attention recently, fueled by applications where decisions evolve on curved spaces. Building on OCO several works have analyzed g-convex functions on Hadamard manifolds, showing regret bounds for first-order and zeroth-order (bandit) ROGD, which depend on both the sectional curvature and the time horizon \cite{wang2023online,sahinoglu2025decentralizedonlineriemannianoptimization}. 
Beyond basic ROGD, optimistic and projection-free variants also attain static or dynamic regret guarantees under g-convexity \cite{hu2023riemannian,wang2025riemannian}, while bandit and tracking settings quantify how curvature perturbs rates in time-varying problems \cite{maass2022tracking}. Extensions also study positively curved spaces via adapted projection arguments, yielding analogous regret bounds under curvature constraints \cite{wang2023online}, and decentralized settings on manifolds with regret bounds that couple the problem geometry with the network connectivity \cite{chen2024decentralized,sahinoglu2025decentralizedonlineriemannianoptimization}. A recurring theme in the prior is that g-convexity introduces curvature-dependent terms, often via comparison inequalities, that can deteriorate the regret relative to the Euclidean case \cite{wang2023online}. Motivated by the outer characterization of convexity, recent work on h-convexity shows curvature-independent convergence rates for subgradient style methods in offline optimization on Hadamard manifolds \cite{criscitiello2025horospherically, lewis2024horoballs}. This suggests a pathway to {\it curvature-independent regret bounds} in {\it online} settings by leveraging Busemann functions, which is the focus of this paper. 

\section{Problem Formulation}
\label{sec:background}

\subsection{Online Optimization}
Euclidean OCO formalizes sequential decision-making under uncertainty in a dynamic environment \cite{hazan2016introduction}. In this setting, a learner interacts with the environment over $T$ rounds. At each round $t \in [T]:= \{1,2,\dots,T\}$ the learner selects a decision $x_t$ from a convex set $\mansubset$, after which the environment reveals a convex loss function 
$f_t : \mansubset \to \mathbb{R}$, and the learner incurs the loss $f_t(x_t)$. The decision-maker has no prior knowledge of $f_t$ when choosing $x_t$, making the process inherently sequential and adaptive. The standard performance measure is regret, which compares the cumulative loss of the learner to that of the best fixed decision in hindsight as
\begin{equation*}
    \regret := \sumt f_t(x_t)- \argmin_{x\in\mansubset}\sumt f_t(x).
\end{equation*}
The goal of OCO is to design algorithms ensuring sublinear regret, i.e., $Reg_T = o(T)$, so that the average per-round loss approaches that of the best fixed decision as $T \to \infty$.

In Riemannian OCO, the decision set $\mansubset \subseteq \man$, where $\man$ is a Riemannian manifold. Then, in the algorithm design, we must replace linear operations with their geometric counterparts via exponential and logarithmic maps. Under g-convexity, algorithms analogous to OGD achieve regret rates similar to the Euclidean setting, but the constants depend on the curvature due to cosine inequality adjustments in the analysis \cite{wang2023online}. Recent work of \cite{criscitiello2025horospherically} introduces h-convexity on Hadamard manifolds, which provides an outer envelope characterization of convexity, using which we will prove regret bounds that are independent of the curvature, narrowing the gap between Euclidean and Riemannian online optimization.


\subsection{Riemannian Optimization}
In this work, we focus on online optimization on Hadamard manifolds, i.e., simply connected, complete Riemannian manifolds whose sectional curvature is non-positive. 
The space is endowed with a smoothly varying inner product $g_x(\cdot,\cdot)$, simply denoted by $\innerprod{\cdot}{\cdot}$ on each tangent space $\tansp{x}$. A key geometric tool is the exponential map $\expm{x} : \tansp{x} \to \man$, which maps a tangent vector $v \in \tansp{x}$ to the endpoint of the geodesic starting at $x$ with initial velocity $v$. Its inverse, the logarithmic map $\text{Log}_x : \man \to \tansp{x}$, returns the initial velocity of the minimizing geodesic from $x$ to a point $y \in \man$. Since $\man$ is a Hadamard manifold, logarithmic map is well-defined and we can denote the Riemannian distance between  $x,y \in \man$ by $d(x,y)$ ($|xy|$ in the proofs). These operations enable natural extensions of gradient-based updates from Euclidean space to curved domains. Let us now discuss different characterizations of convexity in the Riemannian setting.

\textbf{Inner Characterization.} In Euclidean spaces, convexity is defined through the linear structure of $\R^n$. A function $f:\R^n \to \R$
is convex if its restriction to any line segment is convex. On Riemannian manifolds, the lack of a global linear structure precludes the direct use of line segments. Instead, the role of straight lines is played by geodesics, curves with zero intrinsic acceleration. This naturally leads to the notion of g-convexity based on the inner characterization of convexity \cite{criscitiello2025horospherically}, where $f:\man \to \mathbb{R}$ is geodesically convex if its restriction to any geodesic is convex in the Euclidean sense. This definition preserves many fundamental properties of convexity while respecting the intrinsic geometry of the manifold.

\textbf{Outer Characterization.} Another convexity notion on Hadamard manifolds arises based on the outer characterization of convexity. In Euclidean spaces, this notion is expressed through supporting hyperplanes and the representation of convex functions as point-wise suprema of affine functions. On Hadamard manifolds, the natural analogue of affine functions is Busemann functions, which are defined via geodesic rays and capture the asymptotic geometry of non-positively curved spaces. Using these tools, convex sets can be described as intersections of horoballs (the counterparts of Euclidean half-spaces), and convex functions can be represented as envelope suprema of scaled Busemann functions. This leads to the notion of h-convexity, where a function is h-convex precisely when it admits such an outer representation \cite{criscitiello2025horospherically}. This construction generalizes the Euclidean outer characterization of convexity to the Hadamard setting, offering a complementary perspective to g-convexity.

\begin{definition}[Busemann function] For a unit speed geodesic ray $\gamma:[0,\infty) \to \man$, the Busemann function $B_{\gamma}:\man \to \R$ is defined by $B_{\gamma}(x) = \lim_{t\to \infty}(d(\gamma(t),x)-t)$. We can also denote $B_{\gamma}$ as $B_{p,v}$ with $p=\gamma(0)$ and $v=-\gamma'(0)$.
\end{definition}

Formally, h-convexity can be defined using scaled Busemann functions. Given a point $y \in \man$ and a tangent vector $v \in \tansp{y}$, the associated scaled Busemann function $B_{y,v}:\man \to \mathbb{R}$ generalizes affine linear functions in Euclidean space. 

\begin{definition}[h-convex function]
\label{def:h-con}
A function $f:\man \to \R$ is h-convex if for all $y\in \man$, there exists a tangent vector $v \in \tansp{y}$ such that $f(x)-f(y) \geq B_{y,v}(x) ~ \forall x \in \man$.
\end{definition}

Equivalently, $f$ is h-convex if it can be expressed as the pointwise supremum of scaled Busemann functions \cite{criscitiello2025horospherically}. This definition provides an outer characterization of convexity on Hadamard manifolds, in contrast to g-convexity which follows an inner characterization. As shown in \cite{criscitiello2025horospherically}, h-convexity forms a subclass of g-convexity and admits a rich envelope representation that is particularly useful for optimization. Building on this notion, one can further define strong h-convexity.

\begin{definition}[$\mu$-strongly h-convex function]
\label{def:h-str-con}
A function $f:\man \to \R$ is $\mu$-strongly h-convex if for all $y \in \man $ there exists a tangent vector $v\in \tansp{y}$ such that 
\begin{equation}
    f(x)-f(y) \geq Q_{y,v}^{\mu}(x) := -\frac{1}{2\mu}\norm{v}^2+\frac{\mu}{2} d^2(\expm{y}(-\frac{1}{\mu}v),x), \nonumber
\end{equation}
holds $\forall x\in \man$.
\end{definition}

This refinement parallels the role of strong convexity in Euclidean spaces, ensuring not only uniqueness of minimizers but also improved stability and convergence guarantees for optimization algorithms on Hadamard manifolds.

In the analysis of strongly h-convex functions, we introduce an inequality for Hadamard spaces, which is an extension of Stewart's theorem in Euclidean spaces.
\begin{lemma}[Stewart's theorem in Hadamard spaces]
\label{lem:stew}
Let $\Delta abc$ be a geodesic triangle and $p$ be a point on the geodesic segment $bc$. We then have 
\begin{equation}
\label{eq:stew}
    |ab|^2|pc| + |ac|^2|pb| \geq (|pa|^2+ |pb||pc|)|bc|.
\end{equation}
\end{lemma}

\begin{proof}
We apply the cosine law for Hadamard spaces \cite{alimisis2020continuous} twice as follows 
\begin{align*}
|ab|^2 &\geq |pb|^2 + |pa|^2 - 2\innerprod{\logm{p}b}{\logm{p}a} \\
|ac|^2 &\geq |pc|^2 + |pa|^2 - 2\innerprod{\logm{p}c}{\logm{p}a}. 
\end{align*}
Since $p$ lies on the geodesic segment $bc$, we have 
$\logm{p}b/|pb| = -\logm{p}c/|pc|$ and that $|pc|+|pb|=|bc|$. We can then eliminate the above inner product terms by multiplying the inequalities with $|pc|$ and $|pb|$ to obtain Equation \ref{eq:stew}.
\end{proof}

This inequality plays a central role in the analysis of strongly h-convex functions, 
and it is crucial for establishing the logarithmic regret bound. 
\section{Theoretical Analysis}
\label{sec:theoretical}

\subsection{Technical Assumptions} 
We consider the following assumptions, which are standard in online Riemannian optimization.
\begin{assumption}
\label{assum:set}
The set $\mansubset \subseteq \man$ is a g-convex subset of Hadamard manifold $\man$, and the diameter of $\mansubset$ is bounded by $D$, i.e., $d(x,y)\le D ~ \forall x,y \in \mansubset$. 
\end{assumption}

The geodesic convexity of $\mansubset$ is important to ensure that the projection operation $\proj(z):= \argmin_{x\in \mansubset}~d(x,z)$  is single valued and non-expansive, i.e., $d(\proj(z),\proj(y)) \le d(y,z)$ \cite{bacak2014convex}.
\begin{assumption}
 \label{assum:func}
For any $t\in [T]$ the function $f_t$ is $L$-Lipschitz, i.e., $|f_t(x)-f_t(y)|\le L d(x,y)$. 
\end{assumption}
Throughout, we impose the h-convexity (or strong h-convexity) on the function sequence. Convexity is a standard assumption in online optimization \cite{hazan2016introduction}. In Euclidean spaces, h-convexity coincides with g-convexity, whereas on Riemannian manifolds h-convexity imposes stronger conditions. In this work, we make no assumption on the lower bound of sectional curvature. Instead, we rely on h-convexity, which is a stronger condition than g-convexity, and consequently applies to a smaller class of functions. This allows us to obtain tighter regret bounds compared to the class of g-convex functions. 

\subsection{Algorithm}

We now present the Riemannian Online Gradient Descent (ROGD) algorithm \cite{wang2023online}, which is the natural extension of the Euclidean OGD method to the manifold setting.

\begin{algorithm}[t]
    \caption{Riemannian Online Gradient Descent Algorithm}
   \label{alg:rogd}
\begin{algorithmic}

    \STATE {\bfseries Input:} $\mansubset \subseteq \man$, time horizon $T$, step-size $\eta_t$, initial point $x_1$ 
    \FOR{$t=1$ {\bfseries to} $T$}
    \STATE $g_{t} = \rgrad{f_{t}}(x_{t}) $  
    \STATE $x_{t+1} = \proj{(\expm{x_{t}}(-\eta_t g_{t}))}$

    
    \ENDFOR

\end{algorithmic}
\end{algorithm}

The algorithm proceeds in two steps at each iteration. First, the learner  observes the Riemannian gradient $g_t$ of the function $f_t$ at point $x_t$. Second, the learner updates the iterates by moving along the exponential map to remain on the manifold $\man$, and subsequently projects them onto the subset $\mansubset$ to enforce feasibility. Since the set $\mansubset$ is g-convex, the projection oracle returns a unique solution. ROGD has been analyzed for g-convex and strongly g-convex functions \cite{wang2023online}; however, the analysis for h-convex  and strongly h-convex functions has remained open.

\subsection{Main Results}
In this section, we present the first static regret bounds for ROGD under h-convexity and strong h-convexity of functions.

\begin{theorem}[Regret for h-convex functions]
\label{thm:convex}Suppose that Assumptions \ref{assum:set} and \ref{assum:func} hold and  that $f_t$ is h-convex for any $t\in [T]$. Then, Algorithm \ref{alg:rogd} with the fixed step-size $\eta_t =\eta$ guarantees $O(\eta^{-1} +\eta T)$ static regret. In particular, the choice of $\eta=1/\sqrt{T}$ results in $O(\sqrt{T})$ static regret. 
\end{theorem}

\begin{proof}
Using Definition \ref{def:h-con} on the gradient based on Proposition 2 (iii) \cite{criscitiello2025horospherically}, we get
\begin{equation}
   \regret = \sumt f_t(x_t)-f_t(\xopt) \le -\sumt B_{x_t,g_t}(\xopt). 
\end{equation}
Using the identity $B_{x,cv}=cB_{x,v}$ (see the proof of Proposition 12 in \cite{criscitiello2025horospherically}) and applying Lemma 2 of  \cite{criscitiello2025horospherically}, we obtain 
\begin{equation*}
    -B_{x_t,g_t}(\xopt)=-\frac{B_{x_t,\eta g_t}(\xopt)}{\eta} \le \frac{|\tilde{x}_{t+1}x_t|^2+|\xopt x_t|^2-|\xopt x_{t+1}|^2}{2\eta}, 
\end{equation*}
where $\tilde{x}_{t+1} = \expm{x_{t}}(-\eta g_t)$ and $x_{t+1} = \proj{(\tilde{x}_{t+1})}$. We know by the non-expansiveness of the projection that $|x_{t+1}y|\le|\tilde{x}_{t+1}y|$ for any $y \in \mansubset$. Also, $|\tilde{x}_{t+1}x_t|=\eta \norm{g_t} \leq \eta L$. Then, summing over $t\in[T]$ gives the result as
\begin{align*}
    \regret &\le -\sumt B_{x_t,g_t}(\xopt) \le \frac{1}{2\eta} |x_1\xopt|^2 + \frac{\eta}{2} \sumt \norm{g_t}^2.
\end{align*}
\end{proof}


\begin{remark}
Compared to the results in \cite{wang2023online} based on g-convexity, our bound is independent of the sectional curvature of $\mansubset$. 
\end{remark}

The regret bound in Theorem \ref{thm:convex} matches that of Euclidean OCO with no explicit dependence on the curvature. This indicates that the outer characterization of convexity is the primary factor governing the static regret bound in the Riemannian setting. We next establish the static regret bound for $\mu$-strongly h-convex functions based on Definition \ref{def:h-str-con}.
\begin{theorem}[Regret for strongly h-convex functions]\label{thm:strongconvex}
Suppose that Assumptions \ref{assum:set} and \ref{assum:func} hold and that $f_t$ is $\mu$-strongly h-convex for any $t\in[T]$. Then, Algorithm \ref{alg:rogd} with step-size $\eta_t =1/(\mu t)$ guarantees $O(\log(T))$ static regret.
\end{theorem}
\begin{proof}
We begin with Definition \ref{def:h-str-con} to write
\begin{equation}
\label{eq:str1}
f_t(x_t)-f_t(\xopt)\le \frac{1}{2\mu}\norm{g_t}^2-\frac{\mu}{2} d^2(\expm{x_t}(-\frac{1}{\mu} g_t),\xopt).
\end{equation}
Let $x'_t=\expm{x_t}(-\frac{1}{\mu} g_t)$. We can use Lemma \ref{lem:stew} on the geodesic triangle $\Delta\xopt x_t x'_t$ with the choice of $p=\tilde{x}_{t+1}=\expm{x_{t}}(-\eta_t g_t)$ noting that $\eta_t\le \frac{1}{\mu}$ ensures that $\tilde{x}_{t+1}$ belongs to the segment $x_t x'_t$. We then have 
\begin{align*}
|\xopt x'_t|^2 &\geq \frac{(|\xopt \tilde{x}_{t+1}|^2+ |x_t\tilde{x}_{t+1}||\tilde{x}_{t+1}x'_t| )|x_tx'_t| - |\xopt x_t|^2|\tilde{x}_{t+1}x'_t|}{|x_t\tilde{x}_{t+1}|}. 
\end{align*}
Note that $|x_t\tilde{x}_{t+1}|=\eta_t\norm{g_t}$ and $|x_tx'_t|=\frac{1}{\mu}\norm{g_t}$, so $|\tilde{x}_{t+1}x'_t|=(\frac{1}{\mu}-\eta_t)\norm{g_t}$. Replacing these in the above inequality and recalling that $d^2(x'_t,\xopt)=|\xopt x'_t|^2$ by definition, we can simplify Equation \ref{eq:str1} to get
\begin{align*}
f_t(x_t)-f_t(\xopt) &\le \frac{\eta_t}{2}\norm{g_t}^2 + (\frac{1}{2\eta_t}-\frac{\mu}{2})|\xopt x_t|^2 - \frac{1}{2\eta_t}|\xopt\tilde{x}_{t+1}|^2 \nonumber \\
&\le \frac{\eta_t}{2}\norm{g_t}^2 + (\frac{1}{2\eta_t}-\frac{\mu}{2})|\xopt x_t|^2 - \frac{1}{2\eta_t}|\xopt x_{t+1}|^2,
\end{align*}
where we used the non-expansiveness of the projection in the last inequality, i.e., $|\xopt x_{t+1}| \le |\xopt \tilde{x}_{t+1}|$.
Choosing $\eta_t=\frac{1}{\mu t}$ and simplifying the telescoping sum, we get 
\begin{equation*}
    \sumt f_t(x_t)-f_t(\xopt)\le \frac{ L^2}{2\mu} \sumt \frac{1}{t} \le \frac{ L^2}{2\mu}(1+\log(T)).
\end{equation*}
\end{proof}
This results shows that with the introduction of strong h-convexity, an analogue of strong convexity in Euclidean spaces, it becomes possible to achieve curvature-independent logarithmic regret bounds for Hadamard manifolds.

\section{Numerical Experiments}
\label{sec:numerical}


\begin{figure*}[t!]
\begin{center}
\centerline{\includegraphics[width=1.99\columnwidth,height=0.3\textwidth]{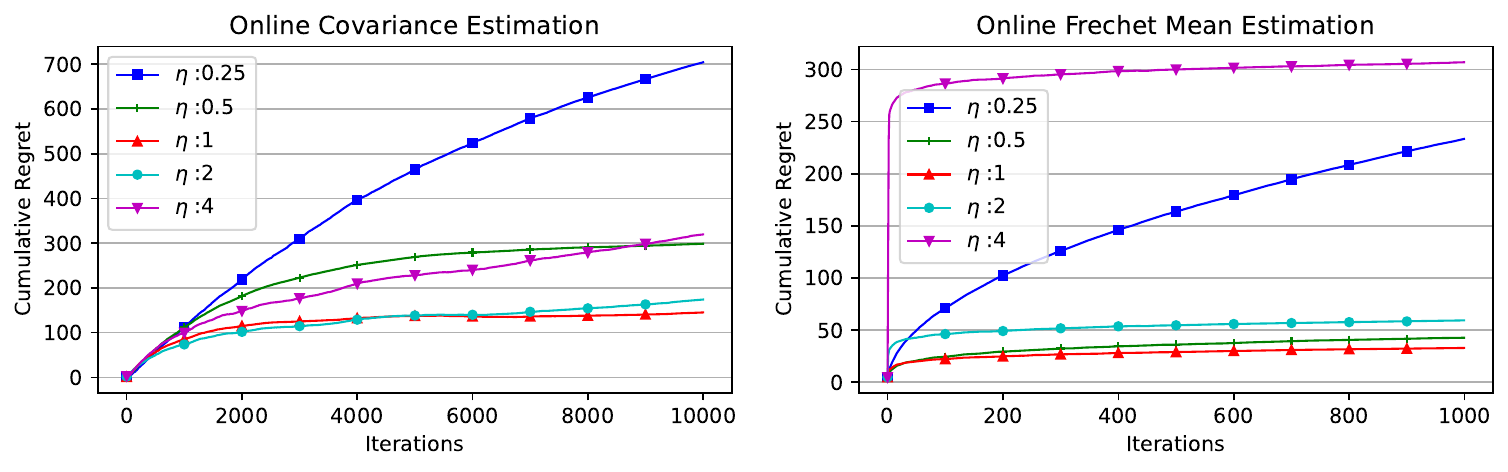}}

\caption{Cumulative regret plots with step-size a) $\eta_t=\frac{\eta}{\sqrt{t}}$ for online covariance estimation  and b) $\eta_t=\frac{\eta}{t}$ for online Fr\'echet mean estimation.}

\label{fig:regret}
\end{center}
\vskip -0.28in
\end{figure*}

We evaluate the performance of ROGD on h-convex objectives. Our experimental setting is based on the manifold of symmetric positive-definite (SPD) matrices, endowed with the affine-invariant Riemannian metric, which is a canonical example of a Hadamard manifold. This space has desirable geometric properties, including non-positive curvature and closed-form expressions for the exponential and logarithmic maps, making it a natural testbed for optimization algorithms. We consider both h-convex and 
$\mu$-strongly h-convex objectives, and design experiments to highlight the impact of these convexity notions on the regret behavior of ROGD. Before presenting our results, we first summarize the fundamental geometric properties of the SPD manifold under the affine-invariant metric, which serves as the foundation for our experimental design.

For the manifold of $n \times n$ SPD matrices, denoted by $\manspd$,  the affine-invariant metric is defined as $ \langle U,V \rangle_X = \mathrm{Tr}\!\left(X^{-1} U X^{-1} V\right)$ for tangent vectors $U,V \in \tanspspd{X}$, where $X \in \manspd$. Under this metric, geodesics between $X,Y \in \manspd$ admit the closed-form expression $\gamma(t) = X^{1/2}\!\left(X^{-1/2} Y X^{-1/2}\right)^t X^{1/2}, \quad t \in [0,1] $, which interpolates smoothly between $X$ and $Y$. The corresponding Riemannian exponential and logarithmic maps are given by $\expm{X}(U) = X^{1/2} \exp\!\left(X^{-1/2} U X^{-1/2}\right) X^{1/2}~\text{and}~\logm{X}(Y) = X^{1/2} \log\!\left(X^{-1/2} Y X^{-1/2}\right) X^{1/2}$, respectively. Moreover, the geodesic distance is expressed as $d(X,Y) = \|\log(X^{-1/2} Y X^{-1/2})\|_F,$ which depends only on the eigenvalues of $X^{-1/2} Y X^{-1/2}$. Importantly, the SPD manifold endowed with this metric has \emph{non-positive sectional curvature}, guaranteeing uniqueness of geodesics and the validity of convexity notions such as g-convexity and h-convexity. These properties make $\manspd$ a particularly well-suited setting for testing online Riemannian optimization algorithms on Hadamard manifolds.

\subsection{Online Tyler's $M$-Estimator on SPD Manifold}
We now evaluate ROGD on an online version of Tyler's $M$-estimator problem to analyze h-convex objectives. In this task, at time $t$ the environment reveals a sample $a_t\in\mathbb{R}^n$, and the learner incurs the time-dependent loss 
$$ f_t(\Sigma)\;=\;\log\big(a_t^\top \Sigma^{-1} a_t\big),\qquad \Sigma\in \manspd,$$
which is the per-sample Tyler objective. The usual Tyler estimator problem minimizes the sum of these terms subject to a scale constraint.  These log-quadratic losses are natural on $\manspd$ and 
fit into the horospherical/outer-envelope framework studied for Hadamard manifolds; in particular, such losses can be related to scaled Busemann functions on the SPD manifold and therefore admit h-convex representations under the affine-invariant geometry \cite{criscitiello2025horospherically}.

For optimization with the affine-invariant metric, the Riemannian gradient has a compact closed-form.  Writing $s_t := a_t^\top \Sigma^{-1} a_t,$ the Euclidean gradient of $f_t$ is $-\Sigma^{-1} a_t a_t^\top \Sigma^{-1}/s_t$, and the affine-invariant Riemannian gradient is $$\operatorname{grad} f_t(\Sigma)\;=\;\Sigma(-\Sigma^{-1} a_t a_t^\top \Sigma^{-1}/s_t)\Sigma\;=\;-\,a_t a_t^\top/s_t.$$
Thus, the Riemannian descent direction $-\operatorname{grad} f_t(\Sigma)$ is the positive semi-definite rank-one matrix $a_t a_t^\top/s_t$. Therefore, using the exact exponential map for the affine-invariant metric, a single ROGD step with step-size $\eta_t>0$ is 
\begin{align*}
\resizebox{\linewidth}{!}{$
   \Sigma_{t+1}
\;=\;
\expm{\Sigma_t}\!\Big(\eta_t\frac{a_t a_t^\top}{s_t}\Big)
\;=\;
\Sigma_t^{1/2}\,\exp\!\Big(\eta_t\,\Sigma_t^{-1/2}\frac{a_t a_t^\top}{s_t}\Sigma_t^{-1/2}\Big)\,\Sigma_t^{1/2}. 
$}
\end{align*}
We implement this update and evaluate the regret trajectory for h-convex objective scenario.  The closed-form expressions above enable efficient and numerically stable implementation of each online iteration on $\manspd$.

In this experiment, we consider the manifold of symmetric positive-definite matrices $\mathrm{SPD}(16)$. We generate $T = 10^4$ samples from a zero-mean Gaussian distribution with covariance matrix $\Sigma_{\text{true}}$, and each sample $a_t$ defines an objective $f_t$. The ROGD algorithm is executed with an adaptive step-size $\eta_t = \eta/\sqrt{t}$ to examine cumulative regret over different time horizons. To assess the sensitivity to the initialization, we select multiple values of $\eta$ from a predefined set $\{0.25,0.5,1,2,4\}$, and we plot the regret against $\Sigma^*$ in Fig. \ref{fig:regret}. The results demonstrate that convergence is sensitive to the choice of the initial step-size, while the algorithm consistently achieves sublinear regret growth of order $O(\sqrt{T})$, in line with our theoretical guarantee in Theorem \ref{thm:convex}.

\subsection{Online Fr\'echet Mean on SPD Manifold}

For strongly h-convex objectives, we consider the online computation of the Fr\'echet mean on the  manifold $\manspd$, equipped with the affine-invariant metric. The protocol is sequential: at each round $t$, the learner chooses a current iterate $\Sigma_t\in\manspd$ (before observing the data); then, the environment reveals a data sample $Y_t\in\manspd$, and the learner incurs the loss $f_t(\Sigma)\;=\;\tfrac{1}{2}\,d^2(\Sigma,Y_t),$
where the squared-distance function is 1-strongly h-convex. The online Fr\'echet mean problem thus seeks to minimize the cumulative squared-distance loss in an adversarial or streaming setting, typically measured via regret against the best fixed SPD matrix in hindsight. For the squared-distance objective we have the well-known Riemannian gradient formula $\operatorname{grad} f_t(\Sigma_t)\;=\;-\,\logm{\Sigma_t}(Y_t),$ where for the affine-invariant metric $$\logm{\Sigma}(X)=\Sigma^{1/2}\log\!\big(\Sigma^{-1/2}X\Sigma^{-1/2}\big)\Sigma^{1/2}.$$
Applying ROGD with step-size $\eta_t>0$ yields the compact geodesic update
\begin{align*}
\Sigma_{t+1}&=\expm{\Sigma_t}\!\big(-\eta_t\operatorname{grad}f_t(\Sigma_t)\big)
=\Sigma_t^{1/2}\!\big(\Sigma_t^{-1/2}Y_t\Sigma_t^{-1/2}\big)^{\eta_t}\!\Sigma_t^{1/2},    
\end{align*}
i.e., the learner moves from $\Sigma_t$ toward $Y_t$ along the geodesic by a fractional step-size $\eta_t$. This closed-form representation makes the online Fr\'echet mean especially convenient for numerical experiments with ROGD on $\manspd$.

In this experiment, we consider the manifold of symmetric positive-definite matrices $\mathrm{SPD}(16)$. A total of $T = 10^3$ samples are generated on the manifold, where each sample $Y_t$ defines an objective $f_t$. The ROGD algorithm is implemented with an adaptive step-size $\eta_t = \eta / t$, reflecting the strong convexity of the objective. To evaluate the sensitivity to the initialization, we test multiple values of $\eta$ chosen from a predefined set $\{0.25,0.5,1,2,4\}$. The results are reported in Fig. \ref{fig:regret}, which shows that the algorithm achieves logarithmic  regret growth, consistent with our theoretical claim in Theorem \ref{thm:strongconvex} and robust to variations in the step-size initialization. Compared to the $O(\sqrt{T})$ behavior observed for merely h-convex objectives, the strong convexity assumption yields significantly faster convergence.

\section{Conclusions}
In this work, we investigated online Riemannian optimization beyond the standard g-convexity framework by analyzing h-convex and strongly h-convex objectives. We studied the ROGD algorithm on Hadamard manifolds and established curvature independent regret bounds, thereby overcoming limitations arising from sectional curvature terms in the g-convex setting. Our numerical experiments on the manifold of SPD matrices with the affine-invariant metric, including online Tyler's estimation and online Fr\'echet mean computation, demonstrated the practical relevance of our approach. These results highlight the potential of h-convexity as a powerful tool for extending online optimization theory to broader non-Euclidean settings. Future directions include extending our theoretical framework to decentralized and stochastic Riemannian optimization, as well as exploring applications in control and large-scale machine learning.


\addtolength{\textheight}{-12cm}   






\bibliographystyle{plain}
\bibliography{refs}

\end{document}